\documentclass[letterpaper, 10 pt, conference]{ieeeconf}  

\IEEEoverridecommandlockouts                              

\usepackage{amsmath} 
\usepackage{amssymb}  
\usepackage{xcolor}
\allowdisplaybreaks
\usepackage{booktabs}
\usepackage{mathtools}
\usepackage{makecell}

\title{\LARGE \bf
Improved Corner Cutting Constraints for Mixed-Integer Motion Planning of a Differential Drive Micro-Mobility Vehicle
}

\author{Angelo Caregnato-Neto$^{1}$ and Janito Vaqueiro Ferreira$^{1}$
	\thanks{This study was financed, in part, by the S\~ao Paulo Research Foundation (FAPESP), Brasil. Process Number 2024/04703-0.}
	\thanks{$^{1}$Angelo Caregnato-Neto and Janito Vaqueiro Ferreira are with Department of Computational Mechanics, School of Mechanical Engineering, University of Campinas, Campinas, Brazil. \tt\small {caregnato.neto@ieee.org}, \tt\small {janito@unicamp.br} }%
}

\newtheorem{theorem}{Theorem} 
\newtheorem{lemma}{Lemma} 

\begin{document}

\maketitle
\thispagestyle{empty}
\pagestyle{empty}

\begin{abstract}

This paper addresses the problem of motion planning for differential drive micro-mobility platforms. This class of vehicle is designed to perform small-distance transportation of passengers and goods in structured environments. Our approach leverages mixed-integer linear programming (MILP) to compute global optimal collision-free trajectories taking into account the kinematics and dynamics of the vehicle. We propose novel constraints for intersample collision avoidance and demonstrate its effectiveness using pick-up and delivery missions and statistical analysis of Monte Carlo simulations. The results show that the novel formulation provides the best trajectories in terms of time expenditure and control effort when compared to two state-of-the-art approaches.
\end{abstract}

\section{INTRODUCTION}

Micro-mobility vehicles aim to satisfy the demand for small-distance transportation of goods and passengers. Idealized mainly as low-velocity platforms that operate in structured environments such as industrial parks and university campuses \cite{micromob1}, the development of autonomous micro-mobility fleets has received extensive attention in the literature \cite{micromobAR,micromob2} and may be fundamental for the development of sustainable cities \cite{micromob3}. Due to their small size and simplicity, these vehicles have also been designed as differential drive platforms \cite{micromobAR} allowing the use of several planning and control techniques originally devised for robots \cite{cones} in the development of their control architectures.

Mixed-Integer Linear Programming (MILP) is a class of mathematical programs where continuous and integer decision variables are optimized considering a linear cost and affine constraints. It has been extensively used for motion planning of a multitude of platforms, such as differential drive robots \cite{cones} and automobiles \cite{quirynen2024real}. MILP is particularly attractive due to its capability to yield global optimal solutions in finite time \cite{richards2005mixed}, handle the nonconvexity of collision avoidance problems, and provide a framework where trajectory planning and decision-making problems can be solved jointly \cite{afonso2020}.

An important issue faced in the early stages of research on  MILP-based motion planners was the prevention of intersample collisions, i.e., although an agent was guaranteed to be outside the boundaries of an obstacle at discrete planning time steps, collisions could still occur if the planned trajectory would cut through an object between such steps. 

Early solutions to this issue required a compromise between computational complexity and conservatism of the solutions. The use of small discretization steps alongside the proper inflation of the obstacle representation within the optimization problem could prevent the occurrence of intersample collisions \cite{kuwata}. However, the computational complexity of the problem increases as these steps become smaller and, conversely, if the steps got larger, the obstacles would have to be further inflated resulting in more conservative solutions or even unfeasibility.
Further progress followed in \cite{kawakami_cornercut} with the development of constraints that guaranteed intersample collision avoidance. The idea was improved in \cite{richardsCornerCut} with an efficient approach that introduces less complexity to the optimization problem while providing identical solutions.  Alternative approaches leveraged the idea of assigning entire segments of the trajectories to predetermined safe regions in the environment \cite{cornercut_tedrake}, effectively preventing intersample avoidance and drastically reducing the complexity of the resulting mixed-integer optimization. Nevertheless, all existing approaches still limit the search space of the optimization and, consequently, can yield conservative solutions.

This paper explores a new condition based on sampled intermediary points of line segments written as the convex combination of its extremities. This scheme was originally devised in \cite{ecc23} for clear line-of-sight connections in a multi-robot coordination problem. We show that the same condition can be readily modified to encode intersample obstacle avoidance constraints. Then, we propose a novel method designed specifically for the problem of differential-drive micro-mobility vehicles and demonstrate through simulation and statistical analysis that the new approach yields more efficient maneuvers at the cost of a modest increase in optimization time.

In this work, we denote sets of subsequent integers from $c$ to $d$ as $\mathcal{I}_c^d = \{c, c+1,\dots, d\}$. A binary implication, e.g., for $b \in \{0,1\}$,  $b \implies `\circ$' is equivalent to $b = 1 \implies '\circ$'. A 1-vector of size $n$ is denoted by $\mathbf{1}_n = [1,1,\dots,1] \in \mathbb{R}^n$.

\section{Problem statement}\label{sec:prob}

Consider a differential-drive vehicle with dynamics discretized with a sampling period $T>0$ and represented by
\begin{align}
	& r_{x}(k+1) = r_{x}(k) +\sigma(k)\cos(\psi(k)), \label{eq:dyn_rx} \\
	& r_{y}(k+1) = r_{y}(k) +\sigma(k)\sin(\psi(k)), \label{eq:dyn_ry}\\
	& \sigma(k) = \xi(k)T + 0.5T^2a(k), \label{eq:dyn_sigma}\\
	& \xi(k+1) = \xi(k) + a(k)T, \label{eq:lin_vel}\\
	& \psi(k+1) = \psi(k) + \Delta \psi(k),\ \forall k \geq 0, \label{eq:dyn_psi}
\end{align}
where $a,\xi,\sigma \in \mathbb{R}$ are the linear acceleration, velocity, and displacement, respectively. The position components in the $x$ and $y$ axes of an inertial coordinate system are represented by $r_x \in \mathbb{R}$ and $r_y \in \mathbb{R}$, respectively. The orientation of the robot is denoted by  $\psi \in [-\pi,\pi]$ and the corresponding increment by $\Delta \psi$. 
For a smooth operation of the vehicle its linear acceleration, velocity, and orientation increment must be constrained to the following bounds: $a \in [a^{\text{min}},a^{\text{max}}]$, $\xi \in [\xi^{\text{min}},\xi^{\text{max}} ]$, and $\Delta \psi \in [\Delta \psi^{\text{min}},\Delta \psi^{\text{max}}]$.

The vehicle operates in a planar region of operation represented by polytope  $\mathcal{A} = \{ \boldsymbol{\alpha} \in \mathbb{R}^2\ \vert \ \mathbf{P}^{\text{ope}} \boldsymbol{\alpha} \leq \mathbf{q}^{\text{ope}} \}$,
	where $\mathbf{P}^{\text{ope}} \in \mathbb{R}^{2\times n_s^{\text{ope}}}$ and $\mathbf{q}^{\text{ope}} \in \mathbb{R}^{n_s^{\text{ope}}}$ with $n_s^{\text{ope}} \in \mathbb{N}$ being the number of sides of $\mathcal{A}$.
	This region contains $n_o \in \mathbb{N}$ obstacles  $\mathcal{O}_c = \{ \boldsymbol{\alpha} \in \mathbb{R}^2\ \vert \ \mathbf{P}^{\text{obs}}_c \boldsymbol{\alpha} \leq \mathbf{q}_c^{\text{obs}} \},\ \forall c \in \mathcal{I}_1^{n_o}$ ,
		where $\mathbf{P}^{\text{obs}}_c \in \mathbb{R}^{2\times n^{\text{obs}}_{s,c}}$ and $n^{\text{obs}}_{s,c} \in \mathbb{N}$ is the number of sides of the $c$-th obstacle. 
		
		We consider a pick-up and delivery problem where the vehicle must collect goods or passengers at a starting region $\mathcal{P}= \{ \boldsymbol{\alpha} \in \mathbb{R}^2\ \vert \ \mathbf{P}^{\text{str}} \boldsymbol{\alpha} \leq \mathbf{q}^{\text{str}} \}$, $\mathbf{P}^\text{str} \in \mathbb{R}^{2 \times n^\text{str}_s}$, with $n^\text{str}_s \in \mathbb{N}$ being the number of sides of region $\mathcal{P}$. Then, the vehicle must deliver the passenger to a series of $n_d \in \mathbb{N}$ destination regions $ \mathcal{D}_t = \{\boldsymbol{\alpha} \in \mathbb{R}^2\ \vert \ \mathbf{P}^\text{des}_t \boldsymbol{\alpha} \leq \mathbf{q}^\text{des}_t \}$, $\forall t \in \mathcal{I}_1^{n_d}$ sequentially, where $\mathbf{P}^\text{des}_t \in \mathbb{R}^{2\times n^\text{des}_{s,t}}$ with $n^\text{des}_{s,t} \in \mathbb{N}$ being the number of sides of region $\mathcal{D}_t$. The objective is to deliver the load to its destinations while performing a smooth and efficient maneuver in terms of time expenditure and control effort. See Fig. \ref{fig:scenario} for an example of the envisioned scenario.
		
		\begin{figure}[!ht]
			\begin{center}	\includegraphics[width=0.40\textwidth]{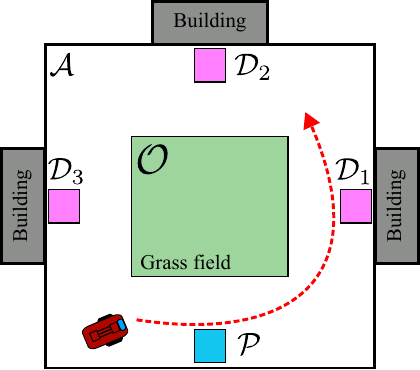}
			\end{center}
			\caption{Example of envisioned scenario: three buildings around a grass field. The vehicle transports passengers from a pick-up region $\mathcal{P}$ around the obstacle $\mathcal{O}$ through the drop regions $\mathcal{D}_i$, $i=1,2,3$ in each building.}
			\label{fig:scenario}
		\end{figure}

\section{Mixed-Integer Motion Planning}

This section presents the  MILP formulation for pick-up and delivery of passengers using a micro-mobility differential drive vehicle. We consider that a maneuver with a maximum horizon of $N^\text{max} \in \mathbb{N}$ time steps must be computed, i.e., the vehicle has a maximum $TN^\text{max}$ seconds to pick up a passenger at region $\mathcal{P}$ and drive him through regions $\mathcal{D}_t,\ \forall t \in \mathcal{I}_1^{n_d}$. The mission is finished after the last destination $\mathcal{D}_{n_d}$ is reached.

\subsection{Prediction model} \label{subsec:pred_model}
Although only linear constraints can be encoded into MILP formulations, the nonlinear dynamics equations  (\ref{eq:dyn_rx}-\ref{eq:dyn_ry}) can be integrated by considering a finite number of  $n^\text{ori} \in \mathbb{N}$ potential vehicle orientations $\boldsymbol{\Psi} = [\bar{\psi}_1,\bar{\psi}_2,\dots, \bar{\psi}_{n^\text{ori}} ]$ and then imposing the following constraints \cite{cones}, $\forall g \in \mathcal{I}_1^{n^\text{ori}}$, $\forall k \in \mathcal{I}_0^{N^\text{max}-1}$, 
\begin{align}
	&	b^\text{ori}_g(k)  \implies \begin{cases}
		r_{x}(k+1) = r_{x}(k) +\sigma(k)\cos(\bar{\psi}_g) \\
		r_{y}(k+1) = r_{y}(k) +\sigma(k)\sin(\bar{\psi}_g) \\
	\end{cases} \label{const:dyn}
\end{align}
and, $\forall k \in \mathcal{I}_0^{N^\text{max}}$,
\begin{align}
	& \sum_{g=1}^{n^\text{ori}} b_g^\text{ori}(k) = 1,  \label{const:ori_bin}
\end{align}
where $b^\text{ori}_g \in \{0,1\}$ is a binary variable associated with the orientation $\bar{\psi}_g$. Constraint (\ref{const:dyn}) imposes that if $b^\text{ori}_g$ is active, then the dynamics (\ref{eq:dyn_rx}-\ref{eq:dyn_ry}) are imposed considering the predetermined angle $\bar{\psi}_g$ and the trigonometric functions (evaluated \textit{a priori}) are constant values, yielding piecewise affine constraints suitable for the MILP formulation. Constraint (\ref{const:ori_bin}) guarantees that only one orientation can be selected at each time step. For the sake of conciseness, we do not discuss the implementation of (\ref{const:dyn}) and refer the reader to \cite{cones} for a thorough presentation. The dynamics (\ref{eq:dyn_sigma}-\ref{eq:dyn_psi}) are represented by linear equations on the optimization values and are encoded directly.

The orientation of the vehicle at time step $k$ is computed as $\psi(k) =\boldsymbol{\Psi}\left[b^\text{ori}_1(k),\dots,b^\text{ori}_{n^\text{ori}}(k)\right]^\top$ and constraints on the orientation increment are written as
\begin{align}
	\Delta \psi^\text{min} \leq \vert \psi(k+1)-\psi(k) \vert \leq \Delta \psi^\text{max}. \label{const:psi_inc}
\end{align}
\subsection{Pick-up and delivery of loads}

Let $\mathbf{r}(k) = [r_x(k),r_y(k)]^\top$ denote the position vector. Then, we employ the ``Big-M" approach to encode constraints that enforce the vehicle to reach the pick-up region,
\begin{align}
	& \mathbf{P}^\text{str} \mathbf{r}(k) \leq \mathbf{q}^\text{str} + (1-b^\text{str}(k))M,\ \forall k \in \mathcal{I}_0^{N^\text{max}}, \label{const:pick_up_region_1}\\
	& \sum_{k=1}^{N^\text{max}} b^\text{str}(k) = 1, \label{const:pick_up_region_2}
\end{align}
where $b^\text{str} \in \{0,1\}$ is a binary variable and $M$ the predetermined ``Big-M" constant. Constraint (\ref{const:pick_up_region_1}) enforces that if $b^\text{str}(k) =1$, then the vehicle must be within the pick-up region $\mathcal{P}$ at time step $k$, whereas (\ref{const:pick_up_region_2}) guarantees that this binary to be activated for at least one time step in the interval $ 0 \leq k \leq N^\text{max}$. Henceforth, constraints similar to (\ref{const:pick_up_region_1}-\ref{const:pick_up_region_2}) will be presented in their equivalent compact form as $b^\text{str}(k) \implies \mathbf{r}(k) \in \mathcal{P},\ \forall k \in \mathcal{I}_o^{N^\text{str}}$. Thus, the constraints for reaching each destination region are written as
\begin{align}
	b^\text{des}_t(k) \implies \mathbf{r}(k) \in \mathcal{D}_t,\ \forall t \in \mathcal{I}_1^{n_d},\ \forall k \in \mathcal{I}_0^{N^\text{max}},
\end{align}
where $b^\text{des}_t \in \{0,1\}$ is the $t$-th destination binary associated with $\mathcal{D}_t$. The order of visitation for $\mathcal{P}$ and $\mathcal{D}_t,\ \forall t \in \mathcal{I}_1^{n_t}$, is established with the following constraints, $\forall k \in \mathcal{I}_0^{N^\text{max}}$,
\begin{align}
	\sum_{n=1}^k b^\text{str}(n) \geq 	\sum_{n=1}^k b_1^\text{des}(n) \geq \dots \sum_{n=1}^k b_{n_t}^\text{des}(n). \label{const:seq_tar}
\end{align}
Constraint (\ref{const:seq_tar}) prevents any destination binary $b^\text{des}_t$ from taking the value of 1 before $b^\text{str}$. Similarly, the destination binary corresponding to  $\mathcal{D}_{t+1}$ can never be activated before the one associated with $\mathcal{D}_t$. Thus, the planned trajectory of the vehicle will reach $\mathcal{P}$ and $\mathcal{D}_1,\dots, \mathcal{D}_{n_t}$ sequentially.

Finally, we establish the terminal condition by determining the requirements for horizon binary $b^\text{hor} \in \{0,1\}$ to be activated:
\begin{align}
	b^\text{hor}(k) \leq \sum_{n=0}^k b^\text{des}_{n_t}(n),\ \forall k \in \mathcal{I}_0^{N^\text{max}} \label{const:terminal}.
\end{align}
Constraint (\ref{const:terminal}) enforces that $b^\text{hor}(k)$ can take the value of 1 ( mission is finished) if the last destination has been reached by the vehicle as indicated by $b^\text{des}_{n_t}$.

\subsection{Collision Avoidance and Conventional Intersample Avoidance}

The prevent collisions with the operational area ($\mathcal{A}$) boundaries, the following constraints are enforced
\begin{align}
	\mathbf{P}^\text{ope} \mathbf{r}(k) \leq \mathbf{q}^\text{ope} + \mathbf{1}_{n^\text{ope}_{s}}\underbrace{M\sum_{n=1}^{k-1}b^\text{hor}(n)}_{\Gamma(k)} \label{const:ope_area}.
\end{align}
Notice that (\ref{const:ope_area}) contains a Big-M relaxation term $\Gamma(k)$ associated with $b^\text{hor}$ which guarantees that the constraints are not enforced for planning steps $k > N$, with $b^\text{hor}(N) = 1$, i.e., constraints are relaxed after the mission is finished.

Considering the polytopic obstacles discussed in Section \ref{sec:prob}, collision avoidance constraints are encoded in the MILP formulation as, $\ \forall k \in \mathcal{I}_0^{N^\text{max}}$, $\forall c \in \mathcal{I}_1^{n_o}$,
\begin{align}
	&-\mathbf{P}^\text{obs}_c \mathbf{r}(k) \leq -\mathbf{q}^\text{obs}_c + (\mathbf{1}_{n^\text{obs}_{s,c}} - \mathbf{b}^\text{obs}_c(k))M  + \mathbf{1}_{n^\text{ope}_{s}}\Gamma(k), \label{const:col_avoid_1}
\end{align}
and, $\forall c \in \mathcal{I}_1^{n_o}$,
\begin{align}
	& \sum_{k=0}^{N^\text{max}} b^\text{obs}_c(k)\geq 1, \label{const:col_avoid_2}
\end{align}
where $\mathbf{b}^\text{obs} \in \{0,1\}^{n^\text{obs}_{s,c}}$ is a vector of binary variables. Constraints (\ref{const:col_avoid_1}) and (\ref{const:col_avoid_2}) guarantee that the vehicle is within at least one external halfspace of each obstacle at all time steps, i.e., outside its boundaries and consequently collision-free.

The imposition of (\ref{const:col_avoid_1}) and (\ref{const:col_avoid_2}) alone does not prevent the intersample collisions illustrated by Fig. \ref{fig:cornercut_example}a, i.e., the problem of trajectories that potentially cut through the obstacle between planning time steps. The following constraint, proposed in \cite{richardsCornerCut} is the general state-of-art solution for this problem:
\begin{align}
	&-\mathbf{P}^\text{obs}_c \mathbf{r}(k+1) \leq -\mathbf{q}^\text{obs}_c + (\mathbf{1}_{n^\text{obs}_{s,c}} - \mathbf{b}^\text{obs}_c(k))M  \nonumber \\ &\qquad \qquad \qquad \qquad \qquad \qquad \qquad \qquad+\mathbf{1}_{n^\text{obs}_{s,c}}\Gamma(k) \label{const:corner_cut_old}.
\end{align}
Constraint (\ref{const:corner_cut_old}) is very similar to (\ref{const:col_avoid_1}), with the only distinction being its imposition over $\mathbf{r}(k+1)$ instead of $\mathbf{r}(k)$. Thus, the obstacle binary vector $\mathbf{b}^\text{obs}_g(k)$ must satisfy the collision avoidance constraint of the vehicle's position at both $k+1$ and $k$. In practice, this is a sufficient condition that enforces $\mathbf{r}(k)$ and $\mathbf{r}(k+1)$ to belong to at least one coincident external halfspace of each obstacle and, consequently, introduces conservatism to the solutions. 
An example is presented in Fig. \ref{fig:cornercut_example}b, where the trajectory must reach an intersection of bottom and right halfspaces before the crossing occurs, requiring an extra time step to do so. Direct transitions (Fig. \ref{fig:cornercut_example}c) are not allowed under constraint (\ref{const:corner_cut_old}).

\begin{figure}[!ht]
	\begin{center}	\includegraphics[width=0.5\textwidth]{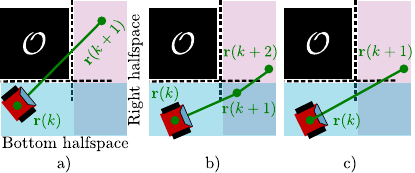}
	\end{center}
	\caption{a) The intersample collision problem. b) A potential solution using (\ref{const:corner_cut_old}) and c) an example of a more efficient maneuver that violates this constraint.}
	\label{fig:cornercut_example}
\end{figure}

\subsection{Intermediary-point (IP) intersample avoidance} \label{subsec:interpoints}

This section demonstrates how the constraint proposed in \cite{ecc23} for line-of-sight connectivity can be modified to guarantee intersample collision avoidance for the vehicle. Let 
\begin{align}
	\mathcal{L}(k) = \{ \boldsymbol{\alpha}\mathbf{r}(k) + (1-\boldsymbol{\alpha})\mathbf{r}(k+1)\ \vert\ \ \forall\boldsymbol{\alpha} \in [0,1 ] \}
\end{align}
be the line segment connecting the position of the vehicle in subsequent steps and described as the convex combination of $\mathbf{r}(k)$ and $\mathbf{r}(k+1)$.   Since differential drive systems move towards their orientation, $\mathcal{L}(k)$ is also the path followed by the vehicle between $k$ and $k+1$. Lemma \ref{lemma} provides conditions for intersample collision avoidance.

\begin{lemma}\label{lemma}
	Let $\mathcal{L}(k)$ be the path of the vehicle connecting its positions $\mathbf{r}(k) \in \mathcal{H}_1$ and $\mathbf{r}(k+1) \in \mathcal{H}_2$, with $\mathcal{H}_1 \subset \mathbb{R}^2$ and $\mathcal{H}_2 \subset \mathbb{R}^2$ being external halfspaces of an obstacle $\mathcal{O}$. If $ \exists \mathbf{z} \in \mathcal{L}\ \vert \  \mathbf{z} \in (\mathcal{H}_1 \cap \mathcal{H}_2)$ then $\mathcal{L}(k) \cap \mathcal{O} = \emptyset$.
\end{lemma}


\begin{proof}
	Consider a point $\mathbf{z} \in \mathcal{L}(k)$ that belongs simultaneously to $\mathcal{H}_1$ and $\mathcal{H}_2$. This point divides $\mathcal{L}(k)$ into two segments $\mathcal{L}_1(k)$ and $\mathcal{L}_2(k)$ such that $\mathcal{L}_1(k) \cup \mathcal{L}_2(k) = \mathcal{L}(k)$. 
	The extremities of $\mathcal{L}_1(k)$ are the points $\mathbf{r}_k \in \mathcal{H}_1$ and $\mathbf{z} \in \mathcal{H}_1$. Similarly, the extremities of $\mathcal{L}_2(k)$ are $\mathbf{z}_k \in \mathcal{H}_2$ and $\mathbf{r}(k+1) \in \mathcal{H}_2$.
	%
	Since all halfspaces of the polygon $\mathcal{O}$ are convex sets, it follows that $\mathcal{L}_1 \subset \mathcal{H}_1$ and $\mathcal{L}_2 \subset \mathcal{H}_2$. By definition of \textit{external} halfspaces it holds that $\mathcal{H}_1 \cap \mathcal{O} = \emptyset$ and $\mathcal{H}_2 \cap \mathcal{O} = \emptyset$ and consequently $\mathcal{L}_1(k) \cap \mathcal{O} = \emptyset$ and $\mathcal{L}_2(k) \cap \mathcal{O} = \emptyset$. Since $\mathcal{L}_1(k) \cup \mathcal{L}_2(k) = \mathcal{L}(k)$, it follows that $\mathcal{L}(k) \cap \mathcal{O} = \emptyset$.
\end{proof}

Following the approach in \cite{ecc23}, we take $n_p \in \mathbb{N}$ equally-spaced intermediary points (IPs) $\{ \mathbf{s}_f(k) \}_{f=1}^{n_p}$, $\mathbf{s}_f(k) \in \mathcal{L}(k),\ \forall f \in \mathcal{I}_1^{n_p}$ and impose the constraints proposed therein to enforce that at least one of such points complies with the conditions present in Lemma \ref{lemma}, $\forall k \in \mathcal{I}_0^{(N^\text{max}-1)}$, $\forall c \in \mathcal{I}_1^{n_o}$,  $\forall f \in \mathcal{I}_1^{n_p}$,
\begin{align}
	-&\mathbf{P}^\text{obs}_c \mathbf{s}_f(k) \leq -\mathbf{q}^\text{obs}_c + (\mathbf{1}_{n^\text{obs}_{s,c}} - \mathbf{b}^\text{obs}_c(k))M \nonumber \\ &\qquad \quad \quad \quad \quad +   (1-b^\text{inter}_{f,c}(k))M + \mathbf{1}_{n^\text{obs}_{s,c}}\Gamma(k), \label{const:cornercut_inter1}\\
	-&\mathbf{P}^\text{obs}_c \mathbf{s}_f(k) \leq -\mathbf{q}^\text{obs}_c + (\mathbf{1}_{n^\text{obs}_{s,c}} - \mathbf{b}^\text{obs}_c(k+1))M \nonumber \\ &\qquad \quad \quad \quad \quad + (1-b^\text{inter}_{f,c}(k))M  +\mathbf{1}_{n^\text{obs}_{s,c}}\Gamma(k) , \label{const:cornercut_inter2}
\end{align}
where $b^\text{inter}_{f,c} \in \{0,1\}$ is a binary vector associated with the $f$-th IP and $c$-th obstacle at the $k$-th time step and 
\begin{align}
	\sum_{f=1}^{n_p} b^\text{inter}_{f,c} = 1,\ \forall k \in \mathcal{I}_0^{N^\text{max}},\  \forall c \in \mathcal{I}_1^{n_o},
\end{align}
guarantees that the former constraints must be simultaneously satisfied with at least one intermediary point. We refer the reader to \cite{ecc23} for a thorough discussion on this formulation.

\subsection {Novel intersample avoidance constraint}\label{subsec:new_cornercut}

This section discusses a novel constraint to address the intersample collision avoidance problem.
Similarly to the IP approach, we consider a vehicle moving through the line segment that connects its subsequent positions $\mathbf{r}(k)$ and $\mathbf{r}(k+1)$. However, in this case, the slope $\hat{\psi}_g$ of the segment is known and the corresponding line equation is 
\begin{align}
	p_y = \tan(\bar{\psi}_g)\left(p_x-r_x(k)\right) + r_y(k)), \label{eq:line_eq}
\end{align}
where $p_x \in \mathbb{R}$ and $p_y \in \mathbb{R}$ being coordinates of points belonging to the line. The line \textit{segment} is then determined by imposing the following bounds
\begin{align}
	& r_x(k) \leq p_x \leq r_x(k+1),\  -\pi/2 < \bar{\psi}_g < \pi/2,\nonumber  \\
	& r_x(k+1) \leq p_x \leq r_x(k),\ \text{otherwise}. \label{eq:line_bounds}
\end{align}
Thus, the path of the vehicle between time steps $k$ and $k+1$ can be written as:
\begin{align}
	\mathcal{P}(k) = \{[p_x,p_y] \in \mathbb{R}^2 \vert\ (\ref{eq:line_eq}),\ (\ref{eq:line_bounds}) \}.
\end{align}

Lemma \ref{lemma} can be employed to enforce that $\mathcal{P}(k)$ does not cross any obstacle. However, the new formulation does not require the prior determination of a finite number of sample points. Instead, a general point $\mathbf{z}_c(k)$ belonging to $\mathcal{P}(k)$ is defined using (\ref{eq:line_bounds}) in terms of a continuous \textit{optimization variable} $\hat{z}_c(k) \in \mathbb{R}$ as, $\forall k \in \mathcal{I}_0^{(N^\text{max}-1)}$, $\forall c \in \mathcal{I}_1^{n_o}$, $\forall g \in \mathcal{I}_{0}^{n^\text{ori}}$,
\begin{align}
	&\mathbf{z}_c(k) = [\hat{z}_c(k), \tan(\bar{\psi}_g)\hat{z}_c(k) - \tan(\bar{\psi_g})r_x(k) +r_y(k) ]^\top,\nonumber \\
	& r_x(k) \leq \hat{z}_c(k) \leq r_x(k+1), -\pi/2 < \bar{\psi}_g < \pi/2 \nonumber \\
	& r_x(k+1) \leq \hat{z}_c(k) \leq r_x(k),\ \text{otherwise}. \label{const:cornercut_aux}
\end{align}
In the special cases of $\psi_g = \pm 90^\circ$ (vertical movement), $\mathbf{z}_c(k)$ is written as
\begin{align}
	&\mathbf{z}_c(k) = [r_x(k), \hat{z}_c(k)]^\top,\nonumber \\
	& r_y(k) \leq \hat{z}_c(k) \leq r_y(k+1),\  \bar{\psi}_g = \pi/2 \nonumber \\
	& r_y(k+1) \leq \hat{z}_c(k) \leq r_y(k),\ \bar{\psi}_g = -\pi/2. \label{const:cornercut_aux2}
\end{align}

Then, constraints similar to (\ref{const:cornercut_inter1}-\ref{const:cornercut_inter2}) are enforced over $\mathbf{z}_c$ instead of the predetermined intermediary points, $\forall k \in \mathcal{I}_0^{(N^\text{max}-1)}$, $\forall c \in \mathcal{I}_1^{n_o}$,
\begin{align}
	-&\mathbf{P}^\text{obs}_c \mathbf{z}_c(k) \leq -\mathbf{q}^\text{obs}_c + (\mathbf{1}_{n^\text{obs}_{s,c}} - \mathbf{b}^\text{obs}_c(k))M \nonumber \\ &\qquad \quad \quad \quad \quad  +   (1-b^\text{ori}_g(k))M + \mathbf{1}_{n^\text{obs}_{s,c}}\Gamma(k) \label{const:cornercut_1}\\
	-&\mathbf{P}^\text{obs}_c \mathbf{z}_c(k) \leq -\mathbf{q}^\text{obs}_c + (\mathbf{1}_{n^\text{obs}_{s,c}} - \mathbf{b}^\text{obs}_c(k+1))M \nonumber \\ &\qquad \quad \quad \quad \quad + (1-b^\text{ori}_g(k))M  +\mathbf{1}_{n^\text{obs}_{s,c}}\Gamma(k).\label{const:cornercut_2}
\end{align}

Although similar to the scheme presented in Section \ref{subsec:interpoints}, the new formulation differs fundamentally from it in the fact that the optimization \textit{searches} for any value $\mathbf{z}_c$ within $\mathcal{P}(k)$ that satisfies the condition of Lemma \ref{lemma}, whereas the IP scheme relies on the evaluation of a finite number of predetermined samples, which has been shown to directly impact the conservatism of the solutions \cite{ecc23}. We note that constraints (\ref{const:cornercut_1}-\ref{const:cornercut_2}) are piecewise affine since the trigonometric functions in (\ref{const:cornercut_aux}) are evaluated \textit{a priori} for all potential orientations, as it is done with constraint (\ref{const:dyn}). Thus, they are suitable for MILP models. 

\begin{theorem}
	Satisfaction of constraints (\ref{const:cornercut_aux}-\ref{const:cornercut_2}) guarantee that the conditions of Lemma \ref{lemma} for intersample collision avoidance of $\mathcal{O}_c$ are met.
\end{theorem}

\begin{proof}
	The Big-M relaxation controlled by the binary $b^\text{ori}$ in (\ref{const:cornercut_1}) and (\ref{const:cornercut_2})  guarantees that these constraints are only enforced to a point $\mathbf{z}_c$ belonging to the proper line segment $\mathcal{P}(k)$ determined in (\ref{const:cornercut_aux}) or (\ref{const:cornercut_aux2}) . 
	Notice that (\ref{const:cornercut_1}) and (\ref{const:cornercut_2}) are the same obstacle avoidance constraints as in (\ref{const:col_avoid_1}), but applied to $\mathbf{z}_c$. Thus, if a value of $\mathbf{z}_c$ satisfies (\ref{const:cornercut_1}), it must belong to at least one external halfspace of $\mathcal{O}_c$ that also contains $\mathbf{r}(k)$. The same reasoning holds for $\mathbf{z}_c$ satisfying (\ref{const:cornercut_2}) and belonging to an external halfspace containing  $\mathbf{r}(k+1)$. Thus, by imposing (\ref{const:cornercut_1}) and (\ref{const:cornercut_2}) together, one enforces $\mathbf{z}_c$ to simultaneously belong to the  halfspaces containing $\mathbf{r}(k)$ and $\mathbf{r}(k+1)$ (their intersection) and the condition of Lemma \ref{lemma} is met for the $c$-th obstacle. 
\end{proof}

\subsection{Optimization problem}\label{subsec:opt_prob}

Let $\boldsymbol{\lambda} = [r_x,r_y,\sigma,\xi,a,\psi,b^\text{ori}_g,b^\text{obs}_c,b^\text{str},b^\text{des}_t,b^\text{hor}]$ denote the vector of optimization variables associated with the problem and $\mathbf{b}^\text{hor} = [b^\text{hor}(1),b^\text{hor}(2),\dots,b^\text{hor}(N^\text{max})]^\top$. Then, the following optimization model is proposed.

\begin{subequations}
	\begin{flalign}
		&\text{\textbf{Pick-up and delivery motion planning problem.}} \nonumber \\
		&\underset{\boldsymbol{\lambda} }{\textrm{min}}\ [1,2,\dots,N^\text{max}] \mathbf{b}^\text{hor}+  \mu \sum_{k=0}^{N^\text{max}}]\vert a(k) \label{cost}\vert  \\
		&\textrm{s.t., } \nonumber \\
		&\text{(\ref{eq:dyn_sigma}-\ref{const:ori_bin}) (kinematics-dynamics)}, \\
		&(\text{\ref{const:psi_inc}) (orientation increment)}, \\
		&  \text{(\ref{const:pick_up_region_1}-\ref{const:terminal}) (pick up and delivery)},\\
		&  \text{(\ref{const:ope_area}-\ref{const:col_avoid_2}) (collision avoidance)}, \\
		& \text{(\ref{const:cornercut_aux}-\ref{const:cornercut_2}) (intersample collision avoidance)}, \\
		&  \xi^{\min} \leq \xi(k) \leq \xi^{\max},\  \forall k \in \mathcal{I}_0^{N^\text{max}}\\
		&  a^{\min} \leq a(k) \leq a^{\max},\  \forall k \in \mathcal{I}_0^{N^\text{max}}.
	\end{flalign}
\end{subequations}

Where the cost (\ref{cost})  is comprised of elements penalizing time expenditure and control effort in terms of the acceleration $a(k)$, with the former being weighted by a constant $\mu >0$ selected by the user.
\section{Results}

We evaluate the new intersample collision avoidance constraint with a pick-up and delivery trajectory planning example and a statistical analysis employing the Monte Carlo method. All tests were performed using a computer equipped with 16 GB of RAM and an Intel$^\text{\textregistered}$ i5-1135G7 (2.40 GHz clock) CPU. The optimization models were built using the Yalmip package \cite{Lofberg2004} and the Multi-Parametric (MPT) toolbox \cite{mpt}.  The Gurobi$^\text{\textregistered}$ \cite{gurobi}  solver version 9.5.2  was used to solve the MILP problems. 

\subsection{Pick-up and delivery simulation}

The scenario considered  contains three destinations $\mathcal{D}_j,\ j = {1,2,3}$ and a pick-up region $\mathcal{P}$ distributed in a $100\times100$ m area (See Fig. \ref{fig:traj1}). The micro-mobility vehicle starts at $k=0$ within $\mathcal{D}_3$, must pick up a load (passengers or goods) in $\mathcal{P}$, and then proceed to reach all three destinations. In the center of the environment, an obstacle $\mathcal{O}_1$ represents a no-entry region (e.g., squares, grass fields, or buildings).

We compare the optimization problem presented in Section \ref{subsec:opt_prob} considering: a) the novel intersample collision method proposed and b) the classical approach from \cite{richardsCornerCut}, i.e., constraint (\ref{const:corner_cut_old}).
 The maximum horizon was selected as $N^\text{max} = 14$ and $T_s = 2$ s. The bounds were chosen as $\Delta \psi^\text{min} = -45^\circ,\ \Delta \psi^\text{max} = 45^\circ$, $\xi^\text{min}=0,\ \xi^\text{max} = 10$ m/s, and $a^\text{min} = -15,\ a^\text{max} = 15$ m/s$^2$. The control effort weight was selected as $\mu = 0.01$. The results are presented in Fig. \ref{fig:traj1}. 

\begin{figure}[!ht]
	\begin{center}	\includegraphics[width=0.43\textwidth]{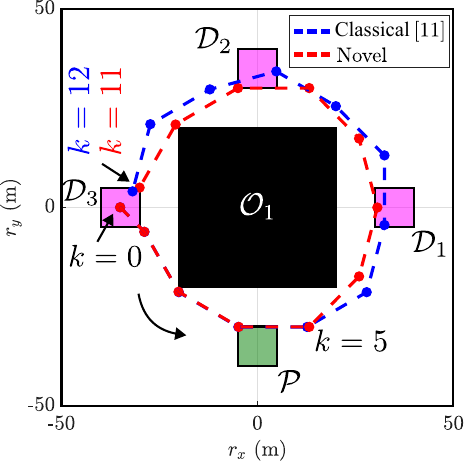}
	\end{center}
	\caption{Comparison between trajectories using novel and classical intersample collision avoidance constraints in a micro-mobility motion pĺanning problem.}
	\label{fig:traj1}
\end{figure}

The resulting trajectories overlap in the initial five time steps. The first substantial difference is observed at $k=6$, as the novel approach is able to yield a direct transition between a position at the bottom external halfspace of $\mathcal{O}_1$ (at $k=5$) to the right-hand side halfspace (at $k=6$), while the classical method places the position of the vehicle at time step $k=6$ in the intersection of these halfspaces. Similar behavior is observed in the following transitions through the corners of $\mathcal{O}_1$. As a result, the circular path given by the novel approach is smaller than its counterpart and the vehicle requires one less time step to reach its destination.

\begin{figure*}[!ht]
	\begin{center}	\includegraphics[width=2\columnwidth]{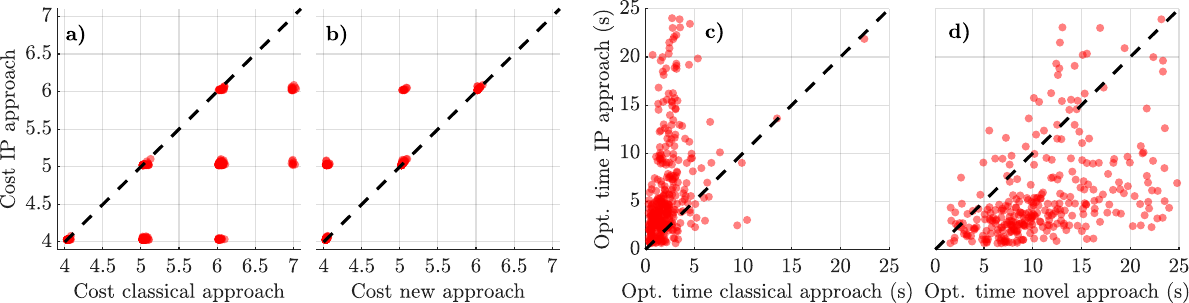}
	\end{center}
	\caption{Comparison of optimal cost: a) IP and \cite{richardsCornerCut}. b) IP and novel method. Comparison of optimization times: c) IP and \cite{richardsCornerCut}. d) IP and novel method. }
	\label{fig:comparison}
\end{figure*}

\subsection{Statistical analysis}

\begin{table}[h!]
	\caption{Performance of the methods. Each metric is evaluated in terms of mean and 95\% bootstrap confidence intervals (CI).}
	\centering
	\begin{tabular}{lccc}
		\toprule
		& & Cost & Optimization Time (s)      \\
		\midrule
		Classical \cite{richardsCornerCut}   & \makecell{Mean \\CI\\ Max.} & \makecell{$5.32$\\ $(5.23,5.40)$  \\$8.02$} & \makecell{$\mathbf{2.61}$\\ $(2.39,2.99)$\\$32.38$}    \\ 
		\midrule
		IP \cite{ecc23} & \makecell{Mean\\CI\\ Max.} & \makecell{$4.59$\\$(4.54,4.66)$\\ $6.09$}  & \makecell{$14.77$\\$(11.96,19.42)$ \\ $482.6$}  \\ 
		\midrule
		Novel & \makecell{Mean\\CI\\ Max.}& \makecell{$\mathbf{4.44}$\\$(4.38,4.50)$\\ $6.09$} & \makecell{$18.94$\\$(17.17,21.92)$\\ $276.5$}   \\
		\bottomrule
	\end{tabular}\label{tab:results_env1}
\end{table}

We also performed a statistical evaluation of the techniques using the Monte Carlo method with 400 randomized scenarios. For the sake of simplicity, we considered the optimization problem presented in Section \ref{subsec:opt_prob} with a single target. For the  tests with the IP scheme, (\ref{const:cornercut_inter1}-\ref{const:cornercut_inter2}) were employed for intersample avoidance with $n_p = 5$.

The vehicle and the target were randomly placed on opposite sides of each environment using a uniform distribution. A random field with 4 to 6 four-sided polygonal obstacles was then generated between the vehicle and the target. 

Table \ref{tab:results_env1} summarizes the results in terms of means and confidence intervals (CI) of the cost and optimization time to reach global optimal solutions.
The classical approach yields worse quality trajectories in terms of the proposed cost (average 5.32). However, it requires substantially less computational time, being the only approach suitable for closed-loop motion planning (receding horizon strategies) where the optimization is repeated periodically \cite{kuwata}. The intermediary points scheme provides better quality solutions, with an average cost of 4.59, at the expense of longer optimization times with an average of 14.77 seconds and a maximum value reaching over 8 minutes. The novel technique provides the best quality trajectories (average cost of $4.44$)  but also demands $22\%$ more computational time than the intermediary points scheme. Nevertheless, since both of these approaches are only suitable for open-loop planning due to the high computational times, the novel formulation is preferred due to the better quality of solutions. Moreover, the maximum time required by the novel approach was $57.3 \%$ shorter than the one required by the IP method.

Fig. \ref{fig:comparison} presents the cost and optimization time data points of the simulations. In the case of the latter, only data points within the neighborhood of the confidence intervals are presented for the sake of better visualization. 	
The costs related to the classical approach were higher or approximately equal to the IP technique for all considered scenarios, as depicted by Fig. \ref{fig:comparison}a. Conversely, Fig. \ref{fig:comparison}b shows that, for the same scenarios, the novel formulation always provided better or approximately equal results when compared to the IP scheme. 

In the case of the optimization times, Fig. \ref{fig:comparison}c shows that in most cases the IP formulation requires more time to compute the global optimal solutions than the classical approach, with few exceptions being present. Fig \ref{fig:comparison}d demonstrates that in most simulations, the IP method required less computational time to finish the optimization than the novel approach, although it must be noted that there are several instances where the latter, which provides equal or better quality solutions, was solved in less time than the former.

\section{CONCLUSIONS} \label{sec:conclusion}

This work explored strategies to guarantee intersample collision avoidance in Mixed-Integer Linear Programming (MILP) motion planning formulations considering a problem of micromobility with differential drive platforms. We demonstrated that with very simple modifications, the intermediary points (IP) approach from \cite{ecc23} can be used for this purpose. A novel formulation was also proposed and compared to the classical scheme from \cite{richardsCornerCut} and the IP one. We concluded that both the novel and IP schemes provide better quality solutions than the one from \cite{richardsCornerCut}, but are unsuitable for closed-loop strategies due to the substantial computational time required. A statistical analysis demonstrated that the novel formulation provides the best trajectories although being more computationally complex, turning it into an attractive method for open-loop planning. Future work should expand the proposed formulation to handle intersample collisions between agents to enable its use in multi-agent motion planning problems. Experiments with  small-scale micro-mobility platforms would further validate the proposal.

\addtolength{\textheight}{-12cm}   

\bstctlcite{IEEEexample:BSTcontrol}
\bibliographystyle{IEEEtran}
\bibliography{IEEEexample}

\end{document}